\def\year{2018}\relax 
\documentclass[letterpaper]{article} 
\usepackage{aaai18} 
\usepackage{times} 
\usepackage{helvet} 
\usepackage{courier} 
\usepackage{url} 
\usepackage{graphicx} 
\usepackage{amsmath}
\usepackage{amssymb}
\usepackage{bm}
\usepackage{xspace}
\usepackage{xcolor}
\usepackage{paralist}
\usepackage{mathrsfs}
\usepackage{amsopn}
\usepackage{enumitem}
\usepackage{pgfcore}
\usepackage{booktabs}
\usepackage{tikz}
\usepackage{floatrow}
\usepackage{stmaryrd}
\usepackage{xspace}
\usepackage{xcolor,colortbl}
\newcolumntype{a}{>{\columncolor{gray!30}}c}

\usepackage{lipsum}
\usepackage{subfig}

\usetikzlibrary{arrows,backgrounds,decorations,decorations.pathmorphing,positioning,fit,automata,shapes,snakes,patterns,plotmarks}
\usepackage{url}
\usepackage{algorithm,algpseudocode}
\usepackage{subfig}
\usepackage[makeroom]{cancel}
\usepackage{todonotes}

\usepackage{amsthm}
\newtheorem{theorem}{Theorem}[section]

\newtheorem{lemma}{Lemma}[section]

\newtheorem{definition}{Definition}[section]

\numberwithin{equation}{section}
\newcommand\vx{\mathbf{x}}
\newcommand\vy{\mathbf{y}}
\newcommand\vz{\mathbf{z}}
\newcommand\vw{\mathbf{w}}
\newcommand\va{\mathbf{a}}
\newcommand\vb{\mathbf{b}}
\newcommand\vc{\mathbf{c}}

\newcommand\vt{\mathbf{t}}
\newcommand\tupleof[1]{\left\langle #1 \right\rangle}

\newcommand\reals{\mathbb{R}}
\newcommand\integers{\mathbb{Z}}
\newcommand\scr[1]{\mathcal{#1}}
\newcommand\grad{\nabla}
\newcommand\sherlock{\textsc{sherlock}\xspace}

\usepackage{balance}

\begin{document}
\title{Output Range Analysis for Deep Neural Networks}

\newcommand{\superscript}[1]{\ensuremath{^{\textrm{#1}}}}

\newcommand{\noop}[1]{}

\newcommand{\poly}{\mathtt{poly}}
\newcommand{\range}{\mathtt{range}}

\author{Souradeep Dutta $^\ast$, Susmit Jha$^\star$, Sriram Sanakaranarayanan$^\dagger$, Ashish Tiwari$^\ddagger$ \\ $^\ast$souradeep.dutta@colorado.edu, $^\star$susmit.jha@sri.com, $^\dagger$ srirams@colorado.edu, $^\ddagger$ tiwari@csl.sri.com }
\maketitle
\vspace*{-0.4cm}
\begin{abstract}
  Deep neural networks (NN) are extensively used for machine learning
  tasks such as image classification, perception and control of
  autonomous systems.  Increasingly, these deep NNs are also been
  deployed in high-assurance applications.  Thus, there is a pressing
  need for developing techniques to verify neural networks to check
  whether certain user-expected properties are satisfied.  In this
  paper, we study a specific verification problem of computing a
  guaranteed range for the output of a deep neural network given a set
  of inputs represented as a convex polyhedron. Range estimation is a
  key primitive for verifying deep NNs.  We present an efficient range
  estimation algorithm that uses a combination of local search and
  linear programming problems to efficiently find the maximum and
  minimum values taken by the outputs of the NN over the given input
  set. In contrast to recently proposed ``monolithic'' optimization
  approaches, we use local gradient descent to repeatedly find and
  eliminate local minima of the function. The final global optimum is
  certified using a mixed integer programming instance. We implement
  our approach and compare it with Reluplex, a recently proposed
  solver for deep neural networks.
  We demonstrate the effectiveness of the proposed approach for verification
  of NNs used in automated control as well as those used in classification. 
\end{abstract}

\section{Introduction}\label{sec:intro}
Deep neural networks have emerged as a versatile and popular
representation model for machine learning due to their ability to
approximate complex functions and the efficiency of methods for
learning these from large data sets.  The black box nature of NN
models and the absence of effective methods for their analysis has
confined their use in systems with low integrity requirements but more
recently, deep NNs are also been adopted in high-assurance systems
such as automated control and perception pipeline of autonomous
vehicles~\cite{kahn2016plato}. While traditional system design
approaches include rigorous system verification and analysis
techniques to ensure the correctness of systems deployed in
safety-critical applications~\cite{Baier+Katoen/2008/Principles}, the
inclusion of complex machine learning models in the form of deep NNs
has created a new challenge to verify these models. 
In this paper, we focus on the \emph{range estimation problem},
wherein, given a neural network $N$ and a polyhedron $\phi(\vx)$
representing a set of inputs to the network, we wish to estimate a
range $\range(l_i, \phi)$ for each of the network's output $l_i$ that
subsumes all possible outputs and is ``tight'' within a given
tolerance $\delta$.  We restrict our attention to feedforward deep
NNs. Furthermore, the NNs are assumed to use only rectified linear
units (ReLUs)~\cite{lecun2015deep} as activation functions.  Adapting
to other activation functions will be discussed in an extended
version.

Despite these restrictions, the range estimation problem has several
applications. Consider a deep NN classification model with the last
layer neural units, before the soft-max computation to determine the
class label, denoted by $l_0, l_1, \ldots, l_k$.  Given an image $i$,
$l_j(i)$ denotes the value of the $j$-th last layer neural unit for
the input $i$ and $\poly(i)$ denote a compact polyhedron region around
$i$ modeling perturbations of the image $i$. The range of label $l_i$
for the polyhedron input $\poly(i)$ is denoted by
$\range(l_i,\poly(i))$.  Label $j$ is possible for a perturbation of
input in $\poly(i)$ if and only if
$(\range(l_0) \times \range(l_1) \times \ldots \times \range(l_k))
\cap \{\forall k \; l_k \leq l_j\}$ is not empty.  Hence, we can
verify if a label $j$ is possible for a given set of perturbations of
image $i$ by solving the range estimation problem. This will establish
the robustness of the deep NN classification models. Another
application of the {range estimation problem} is to prove the safety
of deep NN controllers by proving bounds on the outputs generated by
these models.  This is important because out of bounds outputs can
drive the physical system into undesirable configurations such as the
locking of robotic arm, or command a car's throttle beyond its rated
limits. Finding these errors through verification will enable
design-time detection of potential failures instead of relying on
runtime monitoring which can have significant overhead and also may
not allow graceful recovery. Additionally, range analysis can be
useful in proving the safety of the overall system.

\subsubsection{Related Work} 
The importance of analytical certification methods for neural networks
has been well-recognized in literature.  \cite{kurd2003establishing}
present one of the first categorization of verification goals for NNs
used in safety-critical applications.  The proposed approach here
targets criteria G4 and G5 in \cite{kurd2003establishing} which aim at
ensuring robustness of NNs to disturbances in inputs, and ensuring the
output of NNs are not hazardous.  The verification of neural networks
is a hard problem, and even proving simple properties about them is
known to be NP-complete~\cite{Reluplex}.  The nonlinearity and
non-convexity of deep NNs make their analysis very difficult.  

There
has been few recent results reported for verifying neural networks.  A
methodology for the analysis of ReLU feed-forward networks is studied
in \cite{Reluplex}.  Their approach relies on Satisfiability Modulo
Theory (SMT) solving~\cite{barrett2009satisfiability}, while we only
use a combination of gradient-based local search and MILP
solving~\cite{bixby2012brief}.  The linear programming used for
comparison in Reluplex~\cite{Reluplex} performs significantly less
efficiently according to the experiments reported in this paper. Note,
however, that the scenarios used by Katz et al. are different from
those studied here, and are not publicly available for comparison.  A
related goal of finding adversarial inputs for deep NNs has received a
lot of attention, and can be viewed as a testing approach to NNs
instead of verification method discussed in this paper. A linear
programming based approach for finding adversarial inputs is presented
in \cite{bastani2016measuring}. A related approach for finding
adversarial inputs using SMT solvers that relies on a layer-by-layer
analysis is presented in \cite{HuangKWW16}. The use of SMT solvers for
analysis of NNs has also been studied in \cite{Pulina2012}.  In
contrast, our goal is to not just find adversarial or failing inputs
that violate some property of the output but instead, we aim at
establishing the guaranteed range of the output for a given polyhedral
region of possible inputs.  An abstraction-refinement based iterative
approach for verification of NNs was proposed in
\cite{pulina2010abstraction}. It relies on abstracting NNs as a
Boolean combinations of linear arithmetic SMT constraints that is
guaranteed to be conservative. Spurious counterexamples arising due to
abstraction are used to refine the encoding. Note that our approach
can  fit well inside such an abstraction-refinement
framework.

\subsubsection{Contributions} 

We present a novel algorithm for
propagating convex polyhedral inputs through a feedforward deep neural
network with ReLU activation units to establish ranges s for the
outputs of the network. We have implemented our approach in a tool
called \sherlock. We compare \sherlock with a recently proposed deep
NN verification engine - Reluplex~\cite{Reluplex}.  We demonstrate the
application of \sherlock to establish output range of
deep NN controllers as well as to prove the robustness of deep NN
image classification models.  Our approach seems to scale
\emph{consistently} to randomly generated sparse networks with up to
$250$ neurons and random dense networks with up to $100$ neurons.
Furthermore, we demonstrate its applications on neural networks with
as many as $1500$ neurons.

\section{Preliminaries}\label{sec:background}
We present the preliminary notions including  deep neural networks,
 polyhedra, and mixed integer linear programs.

\subsection{Deep Neural Networks}

We will study feed forward neural networks (NN) using so-called
``ReLU'' units throughout this paper with $n > 0$ inputs and a single
output. Let $\vx \in \reals^n$ denote the inputs and $y \in \reals$ be
the output of the network. Structurally, a NN $\scr{N}$ consists of $k > 0$
hidden layers wherein we assume (for simplicity) that each layer has
$N > 0$ neurons. We use $N_{ij}$ to denote the $j^{th}$ neuron
of the $i^{th}$ layer for $j \in [1,N]$ and $i \in [1,k]$.

A $k$ layer neural network with $N$ neurons in each hidden layer is
described by a set of matrices:
\[ (W_0, \vb_0), \ \ldots\ , (W_{k-1}, \vb_{k-1}), (W_k, \vb_k) \,,\]
wherein
(a) $W_0, \vb_0$ are $ N\times n$ and $N \times 1$ matrices
denoting the weights connecting the inputs to the first hidden layer,
(b) $W_{i}, \vb_i$ for $i \in [1, k-1]$ connect layer $i$ to
layer $i+1$ and (c) $W_k, \vb_k$ connect the last layer $k$
to the output.

\begin{definition}[ReLU Unit]
  Each neuron in the network implements a nonlinear function $\sigma$ linking its
  input value to the output value. In this paper, we consider ReLU units that
  implement the function $\sigma(z):\ \max(z , 0)$.

  We extend the definition of $\sigma$ to apply component-wise to
  vectors $\vz$ as
  $\sigma(\vz): \left(\begin{array}{c} \sigma(z_1)\\ \vdots \\
      \sigma(z_n) \end{array}\right)$.
\end{definition}

Taking $\sigma$ to be the ReLU function, we describe the overall function defined
by a given network $\scr{N}$.

\begin{definition}[Function Computed by NN]
  Given a neural network $\scr{N}$ as described above, the function $F: \reals^n \rightarrow \reals$
  computed by the neural network is given by the composition  $F := F_k \circ \cdots \circ F_0$
  wherein $F_i(\vz): \sigma(W_i \vz + \vb_i)$ is the function computed by the $i^{th}$ hidden layer
  with $F_0$ denoting the function linking the inputs to the first layer and $F_k$ linking the
  last layer to the output.
\end{definition}

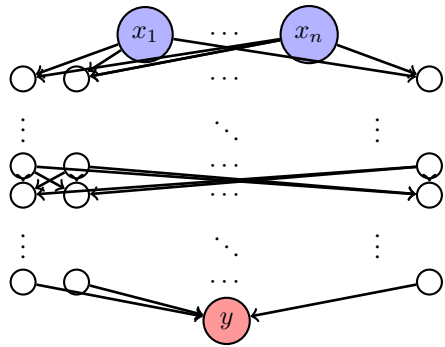
\begin{figure}[t]
  \begin{center}
    \begin{tikzpicture}
      \matrix[netNode/.style={circle, draw=black, thick}, row sep = 0pt, column sep = 10pt]{
        & & \node[netNode,fill=blue!30](n1){$x_1$}; &   \node{$\cdots$};  &    \node[netNode,fill=blue!30](nn){$x_n$}; \\
        \node[netNode](m1){}; & \node[netNode](m2){}; &  & \node{$\cdots$}; &  & &  \node[netNode](mN){}; \\
        \node{$\vdots$}; &   & & \node{$\ddots$}; && \node{$\vdots$}; \\
        \node[netNode](o1){}; & \node[netNode](o2){}; &  & \node{$\cdots$}; &  &  & \node[netNode](oN){}; \\
        \node[netNode](p1){}; & \node[netNode](p2){}; &  & \node{$\cdots$}; &  &  & \node[netNode](pN){}; \\
        \node{$\vdots$}; &   & & \node{$\ddots$}; && \node{$\vdots$}; \\
        \node[netNode](q1){}; & \node[netNode](q2){}; &  & \node{$\cdots$}; &  &  & \node[netNode](qN){}; \\
        & & & \node[netNode,fill=red!40](o){$y$}; \\
      };
      \path[->, line width=1pt] (n1) edge (m1)
      (nn) edge (m1)
      (nn) edge  (m2)
      (n1) edge (m2)
      (nn) edge (m2)
      (n1) edge  (mN)
      (nn) edge  (mN)
      (o1) edge (p1)
      (o1) edge (p2)
      (o1) edge (pN)
      (o2) edge (p1)
      (o2) edge (p2)
      (o2) edge (pN)
      (oN) edge (p1)
      (oN) edge (p2)
      (oN) edge (pN)
      (q1) edge (o)
      (q2) edge (o)
      (qN) edge (o);
      \end{tikzpicture}
  \end{center}
  \caption{Feedforward Neural Network with ReLU Units.}\label{Fig:feedforward-network-schema}
\end{figure}

For a fixed input $\vx$, we say that a neuron $N_{ij}$ is activated if
the value input to it is nonnegative. This corresponds to the output
of the ReLU unit being the same as the input. It is easily seen that
the function $F$ computed by a NN $\scr{N}$ is continuous and
piecewise affine, and differentiable almost everywhere in $\reals^n$.
If it exists, we denote the gradient of this function
$\grad F: ( \partial_{x_1}F,\ \ldots,\ \partial_{x_n}F)$. Computing
the gradient can be performed efficiently (as described subsequently).

\subsection{Mixed Integer Linear Programs}

Throughout this paper, we will formulate linear optimization problems
with integer variables. We briefly recall these optimization problems,
their computational complexity and solution techniques used in practice.

\begin{definition}
  A mixed integer linear program (MILP) involves
  a set of real-valued variables $\vx$ and integer valued variables $\vw$
  of the following form:
  \[ \begin{array}{rcl}
       \max & \va^T \vx + \vb^T \vw \\
       \mathsf{s.t.} & A \vx + B \vw \leq \vc \\
            & \vx \in \reals^n\\
            & \vw \in \integers^m\\
       \end{array}\]
   \end{definition}

   The problem is called a linear program (LP) if there are no integer
   variables $\vw$. The special case wherein $\vw \in \{0,1\}^m$ is
   called a binary MILP. Finally, the case without an explicit
   objective function is called a MILP feasibility problem.
   It is well known that MILPs are NP-hard problems: the best known
   algorithms, thus far, have exponential time complexity in the worst
   case. In contrast, LPs can be solved efficiently using interior
   point methods.

\section{Problem Definition}\label{sec:problem}
Let $\scr{N}$ be a neural network with $n$ inputs $\vx$, a single
output $y$ and weights $\tupleof{(W_0,\vb_0), \ldots, (W_k,
  \vb_k)}$. Let $F_N$ be the function defined by such a network.

\begin{definition}[Range Estimation Problem]
The range estimation problem is defined as follows:
\begin{itemize}
  \item\textsc{Inputs:} Neural network $\scr{N}$, input constraints $P: A \vx \leq \vb$ and a tolerance parameter $\delta > 0$.
  \item\textsc{Output:} An interval $[\ell, u]$ such that
  $ (\forall\ \vx \in P)\ F_N(\vx) \in [\ell, u]$.
  I.e, $[\ell,u]$ contains the range of $F_N$ over inputs $\vx \in P$.
  Furthermore, we require the interval to be \emph{tight}:
  \[ (\max_{\vx \in P}\ F_N(\vx) \geq u - \delta ),\ (\min_{\vx \in P}\ F_N(\vx) \leq \ell + \delta) \,.\]
\end{itemize}
\end{definition}

We will assume that the input polyhedron $P$ is compact: i.e, it is
closed and has a bounded volume.

\subsection{Overall Approach}

Without loss of generality, we will focus on estimating the upper
bound $u$. The case for the lower bound will be entirely
analogous. First, we note that a single MILP can be used to directly
compute $u$, but can be quite expensive in practice.  Therefore, our
approach will combine a series of MILP feasibility
problems alternating with local search steps.

\begin{figure}[t]
\begin{center}
  \begin{tikzpicture}[scale=1.0]
  \draw[fill=red!50, fill opacity=0.3, draw=red!50] (0,0) rectangle (8,0.85);
  \draw[fill=red!20, fill opacity=0.3, draw=red!20] (0,0.85) rectangle (8,2.43);

  \draw[line width=1.0, blue] plot[smooth, tension=0.5] coordinates {(0,0) (0.4,0.2) (1.0,0.4) 
                                                                      (1.4,0.7) (2,0.8) (2.2,0.5) (3,1.2) (4.0, 2.4) (4.2,2.2) 
                                                                      (4.4,2.0) (4.6,1.7) (4.8,1.5) (5,1.2) (6,4) (6.5,3.5) (7,2.8) (8,2.0) };

  \node[circle,fill=black,inner sep=2pt, minimum size=0pt](n0) at (0.4,0.2) {};
    \node[circle,fill=black,inner sep=2pt, minimum size=0pt](n2) at (1.4,0.7) {};
  \node[circle,fill=black,inner sep=2pt, minimum size=0pt](n3) at (2.0,0.8) {};
  \node[circle,fill=black,inner sep=2pt, minimum size=0pt](n4) at (4.8,1.5) {};
  \node[circle,fill=black,inner sep=2pt, minimum size=0pt](n5) at (4.4,2.0) {};
  \node[circle,fill=black,inner sep=2pt, minimum size=0pt](n7) at (4.0,2.4) {};
  \node[circle,fill=black,inner sep=2pt, minimum size=0pt](n8) at (7,2.8) {};
  \node[circle,fill=black,inner sep=2pt, minimum size=0pt](n9) at (6.5,3.5) {};
  \node[circle,fill=black,inner sep=2pt, minimum size=0pt](n10) at (6,4) {};

  \node(m0) at (0.4,-0.2) {$\vx_0$};
  \node(m2) at (1.4,-0.2) {$\vx_1$};
  \node(m3) at (2.0,-0.2) {$\vx_2$};
  \node(m4) at (4.8,-0.2) {$\vx_3$};
  \node(m5) at (4.4,-0.2) {$\vx_4$};
  \node(m7) at (4.0,-0.2) {$\vx_5$};
  \node(m8) at (7,-0.2) {$\vx_6$};
  \node(m9) at (6.5,-0.2) {$\vx_7$};
  \node(m10) at (6,-0.2) {$\vx_8$};

  \path[dotted,thin] (m0) edge (n0)
  (m2) edge (n2)
  (m3) edge (n3)
  (m4) edge (n4)
  (m5) edge (n5)
  (m7) edge (n7)
  (m8) edge (n8)
  (m9) edge (n9)
  (m10) edge (n10);

  \path[->, thick, black,dashed] (n0) edge[bend left] node[above,black]{$L_1$} (n2)
  (n2) edge [bend right] node[above,black]{$L_2$} (n3)
  (n4) edge [bend right] node[right,black]{$L_3$} (n5)
  (n5) edge [bend right] node[right,black]{$L_4$} (n7)
  (n8) edge [bend right] node[right,black]{$L_5$} (n9)
  (n9) edge [bend right] node[right,black]{$L_6$} (n10);

  \filldraw [draw=white,thick,pattern=north west lines, pattern color=red](0,4.2) rectangle (8,4.5);
  \draw[thick,red] (0,4.2)-- (8,4.2);
  \node at (-0.2,4.2) {$u^*$};
  \node at (-0.2,0.8) {$u_2$};
  \node at (-0.2,2.4) {$u_1$};

  \path[->, black, thick, dashed] (n3) edge node[below]{$G_1$} (n4)
  (n7) edge node[above]{$G_2$} (n8);


  \draw[->, line width=1.5] (0,0) -- (8.2,0);
  \draw[->, line width=1.5] (0,0) -- (0,4.5);
  \end{tikzpicture}
\end{center}
\caption{A schematic figure showing our approach showing alternating
  series of local search $L_1, \ldots, L_6$ and ``global search''
  $G_1, G_2$ iterations. The points $\vx_2, \vx_5, \vx_8$ represent local
  minima wherein our approach transitions from local search iterations to global search iterations. }\label{Fig:approach-glance}
\vspace*{-0.5cm}
\end{figure}
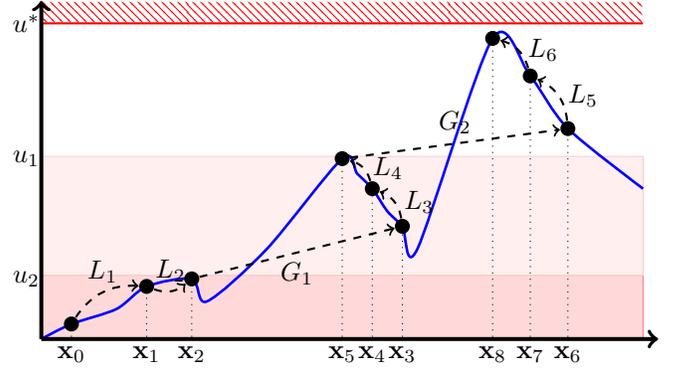

Figure~\ref{Fig:approach-glance} shows a pictorial representation of
the overall approach.  The approach incrementally finds a series of
approximations to the upper bound $ u_1 < u_2 < \cdots < u^* $,
culminating in the final bound $u = u^*$.

\begin{enumerate}
  \item The first level $u_1$ is found by choosing a randomly sampled point $\vx_0 \in P$.
  \item Next, we perform a series of local iteration steps resulting in
  samples $\vx_1, \ldots, \vx_{i}$ that perform gradient ascent
  until these steps cannot obtain any further improvement. We take $u_1 = F_N(\vx_i)$.
  \item A  ``global search'' step is now performed to check if there is any point $\vx \in P$ such that $F_N(\vx) \geq u_1 + \delta$.
     This is obtained by solving a MILP feasibility problem.
  \item If the global search fails to find a solution, then we declare $u^* = u_1 + \delta$.
  \item Otherwise, we obtain a new \emph{witness} point $\vx_{i+1}$ such that $F_N(\vx_{i+1}) \geq u_1 + \delta$.
  \item We now go back to the local search step.
\end{enumerate}

\begin{algorithm}
  \caption{Estimate maximum value $u$ for a neural network $\scr{N}$ over a range $\vx \in P$ with tolerance $\delta > 0$.}\label{Algo:overall-algorithm}
  \begin{algorithmic}[1]
  \Procedure{FindUpperBound}{$\scr{N}$, $P$, $\delta$}
  \State $\vx\ \leftarrow$\ \textbf{Sample}($P$) \label{line:initial-sample}
  \State terminate $\leftarrow$ false
  \While{not terminate}\label{line:while-loop}
     \State $(\vx, u)\ \leftarrow\ $ \textbf{LocalSearch}($\scr{N},\ \vx,\ P$) \label{line:local-search}
     \State $u\ \leftarrow\ u + \delta$ \label{line:estimate-increment}
     \State $(\vx', u', feasible)\ \leftarrow$ \textbf{GlobalSearch}($\scr{N},\ u,\ P$) \label{line:global-search}
     \If{feasible} \label{line:check-feasibility}
      \State $(\vx, u)\ \leftarrow\ (\vx',u')$ \label{line:update-estimates}
     \Else
      \State terminate $\leftarrow$ true
      \EndIf
  \EndWhile
  \State \Return $(\vx, u)$
  \EndProcedure
\end{algorithmic}
\end{algorithm}

Algorithm~\ref{Algo:overall-algorithm} presents the pseudocode for the
overall algorithm.  The procedure starts by generating a sample from
the interior of $P$ (line~\ref{line:initial-sample}) using an
interior-point LP solver. Next, we iterate between local search
iterations (line~\ref{line:local-search}) and global search
(line~\ref{line:global-search}). Each local search call provides a new
point $\vx \in P$ and $u = F_N(\vx)$ that is deemed a local minimum of
the search procedure.  We increment the current estimate $u$ by
$\delta$, the given tolerance parameter
(line~\ref{line:estimate-increment}) and ask the global solver to find
a new point $\vx'$ such that $u' = F_N(\vx')$ is at least as large as
the current value of $u$ (line~\ref{line:global-search}). If this
succeeds, then we update the current values of $\vx, u$
(line~\ref{line:update-estimates}). Otherwise, we conclude that the
current value of $u$ is a valid upper bound within the tolerance
$\delta$.

We assume that the procedures \textbf{LocalSearch} and
\textbf{GlobalSearch} satisfy the following properties:
\begin{itemize}
\item \textbf{(P1)} Given a point $\vx \in P$, the \textbf{LocalSearch} procedure
  returns  $\vx' \in P$ such that $F_N(\vx') \geq F_N(\vx)$.

\item \textbf{(P2)} Given a neural network $\scr{N}$ and the current estimate $u$, the \textbf{GlobalSearch}
procedure either declares \texttt{feasible} along with $\vx' \in P$
such that $F_N(\vx') \geq F_N(\vx)$, or declares \texttt{not feasible}
if no $\vx' \in P$ satisfies $F_N(\vx') \geq F_N(\vx)$.
\end{itemize}

We recall the basic assumptions thus far: (a) $P$ is compact, (b) $\delta > 0$ and (c) properties
\textbf{P1, P2} apply to the \textbf{LocalSearch} and \textbf{GlobalSearch} procedures.
Let us denote the ideal upper bound by $u^*:\ \max_{\vx \in P} F_N(\vx)$.
\begin{theorem}
  Algorithm~\ref{Algo:overall-algorithm} always terminates. Furthermore, the
  output $u$ satisfies $ u \geq u^*\ \mbox{and}\ u \leq u^*+\delta$.
\end{theorem}
\begin{proof}
  Since $P$ is compact and $F_N$ is a continuous
  function. Therefore, the maximum is always attained.

    To see why it terminates, we note that the value of $u$ increases by at least
    $\delta$ each time we execute the loop body of the While loop in line~\ref{line:while-loop}.
    Furthermore, letting $u_0$ be the value of $u$ attained by the sample obtained in line~\ref{line:initial-sample},
    we can upper bound the number of steps by $\left\lceil\frac{(u^*-u_0)}{\delta}\right\rceil$.

    We note that the procedure terminates only if \textbf{GlobalSearch} returns infeasible.
    Therefore, appealing to property \textbf{P2}, we note that
    $ (\forall\ \vx \in P)\ F_N(\vx) \leq u$.
    Or in other words, $u^* \leq u$.

    Likewise, consider the value of $u$ returned by the call to \textbf{LocalSearch} in the last
    iteration of the loop, denoted as $u_n$ along with the point $\vx_n$. We note that $F_N(\vx_n) = u_n$.
    Therefore, we have $ u_n \leq u^* \leq u$.
    However, $u_n = u - \delta$. This completes the proof.
\end{proof}

\section{Local and Global Search}\label{sec:solution}
In this section, we will describe the local and global search algorithms used in our approach
for estimating upper bounds. The approach for estimating lower bounds is identical replacing
ascent steps with descent steps.

\subsection{Local Search Technique}

The local search algorithm uses a steepest projected gradient ascent
algorithm, starting from an input sample point $\vx_0 \in P$ and
$u = F_N(\vx_0)$, iterating through a sequence  of points
$(\vx_0, u_0), (\vx_1, u_1), \ldots, (\vx_n, u_n)$, such that
$\vx_i \in P$ and $u_0 < u_1 < u_2 < \cdots < u_n$.

The new iterate $\vx_{i+1}$ is obtained from  $\vx_i$ as follows:
\begin{compactenum}
\item Compute the gradient $\vz_i:\ \grad F_N(\vx_i)$.
\item Compute a ``locally active region'' $\scr{L}(\vx_i)$.
\item Solve a linear program (LP) to compute $\vx_{i+1}$.
\end{compactenum}

We now describe the calculation of $\vz_i$, the definition of a
``locally active region'' and the setup of the LP.

\paragraph{Gradient Calculation:} 
Technically, the gradient of $F_N(\vx)$ need not exist for each input
$\vx$. However, this happens for a set of points of measure $0$, and
is dealt with in practice by using a smoothed version of the function
$\sigma$ defining the ReLU units.

The computation of the gradient uses the chain rule to obtain the gradient
as a product of matrices:
\[ \vz:\ J_0 \times J_1 \times \cdots \times J_k \,,\]
wherein $J_i$ represents the Jacobian matrix of partial derivatives
of the output of the $(i+1)^{th}$ layer $\vz_{i+1}$ with respect to those of the
$i^{th}$ layer $\vz_i$.  Since $\vz_{i+1} = \sigma (W_i \vz_i + \vb_i)$ we
can compute the gradient $J_i$ using the following simple rule:
\begin{compactenum}
\item If the $j^{th}$ entry $\vz_{i+1,j} \geq 0$ then copy  $j^{th}$ row of $W_i$
to $J_i$.
\item Otherwise, set the $j^{th}$ row of $J_i$ to be all zeros.
\end{compactenum}
In practice, the gradient calculation can be \emph{piggybacked} onto the function
evaluation $F_N(\vx)$ so that we obtain both the output $u$ and the gradient
$\grad F_N(\vx)$.

First order optimization approaches present numerous rules such as the
Armijo step sizing rules for calculating the step
size~\cite{Luenberger/1969/Optimization}. However, in our approach we
exploit the local linear nature of the function $F_N$ around the
current sample input $\vx$ to setup an LP.

\begin{definition}[Locally Active Region]
  For an input $\vx$ to the neural network $\scr{N}$, the locally
  active region $\scr{L}(\vx)$ describes the set of all inputs $\vx'$
  such that $\vx'$ activates exactly the same neurons as $\vx$.
\end{definition}

Given the definition of a locally active region, we obtain the following property.

\begin{lemma}
  The region $\scr{L}(\vx)$ is described by a polyhedron with possibly
  strict inequality constraints.  If $\vx' \in \scr{L}(\vx)$, then
  $\grad F_N(\vx) = \grad F_N(\vx')$.
\end{lemma}
\begin{proof} (Sketch).  To see why $\scr{L}(\vx)$ is a polyhedron, we
  proceed by induction on the number of hidden layers at the
  network. Let $HN_1$ be the subset of neurons in the first hidden
  layer that are activated by $\vx$. We can write linear inequalities
  that ensure that for any other input $\vx'$, the inputs to the
  neurons in $HN_1$ are $ \geq 0$ and likewise the inputs to the
  neurons not activated in the first layer are $ < 0$. Proceeding
  layer by layer, we can write a series of constraints that describe
  $\scr{L}(\vx)$. Since the gradient is entirely dictated by the set
  of active neurons, it is clear that
  $\grad F_N(\vx) = \grad F_N(\vx')$.
\end{proof}

Let $\overline{\scr{L}}(\vx)$ denote the closure of the local active set
by converting the strict $>$ constraints to their non-strict $\geq$ versions.
Therefore, the local maximum is simply obtained by solving the following LP.
\[ \max \  \vw^T \vy \ \mathsf{s.t.}\ 
          \vy \in \overline{\scr{L}}(\vx) \cap P \,.\]
 The solution of the LP above yields a step of the local
 search.

 \paragraph{Termination:} The local search is terminated when each
 step no longer provides a sufficient increase, or alternatively the
 length of each step is deemed too small. These are controlled by user
 specified thresholds in practice. Another termination criterion
 simply stops the local search when a preset maximum number of
 iterations is exceeded. In practical implementations, all three
 criteria are used.

We note that local search described thus far satisfies property \textbf{P1} recalled below.
\begin{lemma}
  Given a point $\vx \in P$, the \textbf{LocalSearch} procedure
  returns a new $\vx' \in P$, such that $F_N(\vx') \geq F_N(\vx)$.
\end{lemma}

\subsection{Global Search}

We will now detail the global search procedure.  The overall goal of
the global search procedure is to search for a point $\vx \in P$ such
that $F_N(\vx) \geq u$ for a given estimate $u$ of the current upper
bound.  The approach formulates a MILP feasibility problem, whose
real-valued variables are:

\begin{compactenum}
\item $\vx$: the inputs to the network with $n$ variables.
\item $\vz_1, \ldots, \vz_{k-1}$, the outputs of the hidden layer. Each $\vz_i$ is a vector
of $N$ variables.
\item $y$: the overall output of the network.
\end{compactenum}

Additionally, we introduce binary ($0/1$) variables
$\vt_1, \ldots, \vt_{k-1}$, wherein each vector $\vt_i$ has the same
size as $\vz_i$. These variables will be used to model the piecewise
behavior of the ReLU units.

Now we encode the constraints. The first set of constraints ensure that
$ \vx \in P$. Suppose $P$ is defined as $A \vx \leq \vb$ then we simply
add these as constraints.

We wish to add constraints so that for each hidden layer $i$,
$\vz_{i+1} = \sigma(W_i \vz_i + \vb_i)$ holds. Since $\sigma$ is not
linear, we use the binary variables $\vt_{i+1}$ to encode the same
behavior:
\[ \begin{array}{rl}
    \vz_{i+1} &  \geq  W_i \vz_i + \vb_i, \\
		 \vz_{i+1} & \leq W_i \vz_i + \vb_i + M \vt_{i+1}, \\
    \vz_{i+1}  & \geq 0, \\
      \vz_{i+1} & \leq M(\mathbf{1} - \vt_{i+1}) \\
   \end{array} \]
 
 Note that for the first hidden layer, we simply substitute $\vx$ for
 $\vz_0$.  This trick of using binary variables to encode piecewise
 linear function is standard in
 optimization~\cite[Ch.~22.4]{Vanderbei/2004/Linear}~\cite[Ch.~9]{Williams/2013/Model}.
 Here $M$ is taken to be a very large constant as a placeholder for
 $\infty$.  Additionally, we can derive fast estimates for $M$ by
 using the norms $||W_i||_{\infty}$ and the bounding box of the input
 polyhedron.

The output $y$ is constrained as: $y = W_{k} \vz_k + \vb_k$.
Finally, we require that $y \geq u$.

The MILP, obtained by combining these constraints, is a feasibility
problem without any objective.

\begin{lemma}
The MILP encoding is feasible if and only if there is an input $\vx \in P$ such
that $y = F_N(\vx) \geq u$.
\end{lemma}


\section{Experimental Evaluation}\label{sec:exp}
We first describe our C++ implementation of the ideas described thus
far, called \sherlock.  \sherlock combines local search with the
commercial MILP solver Gurobi, freely available for academic
use~\cite{gurobi}.  The tolerance parameter $\delta$ is fixed to $10^{-3}$.

For comparison with Reluplex, we used the proof of concept
implementation available at \cite{reluplex_impl}. However, at its
core, Reluplex does not compute intervals. Therefore, we use a
bisection scheme to repeatedly query Reluplex with ever smaller
intervals until, we have narrowed down the actual range within a
tolerance $\hat{\delta} = 10^{-2}$. 
\begin{table*}[t]
  \caption{Performance on randomly generated neural networks. \textbf{Legend:} $n$: \# inputs, $k$: \# hidden layers, $N$: \# neurons/layer, $s$:
    sparsity fraction, $N_c$: successfully solved (out of 100), $T$: running time, \# Iters: number of iterations. The number $(*m)$ denotes
    that Reluplex solved $m$ instances that were not solved by our approach. All experiments were run on a Linux workstation running
    Ubuntu 17.04 with $64$ GB RAM and $23$ cores.}\label{tab:microbenchmarks-1}
  \begin{tabular}{|cccc|a |aaa| aaa|c| ccc|}
    \hline
    \multicolumn{4}{|c|}{Network Params.} & \multicolumn{7}{a|}{\sherlock} & \multicolumn{4}{c|}{Reluplex}\\
    \hline
    $n$ & $k$ & $N$ & $s$ & $N_c$ & \multicolumn{3}{a|}{$T$} & \multicolumn{3}{a|}{\# Iters} & $N_c$ &   \multicolumn{3}{c|}{$T$} \\
    \cline{6-11}\cline{13-15}
        &     &     &     &       &  avg  & min & max & avg & min & max &   & avg & min & max \\
    \hline
2 &	2 &	10 &	0.5 &	100 & 	0.47&   0.0 &  1.95 & 3.5 & 2& 17 & {66} & 2.5 &  1.0 &  2.6 \\[3pt]
3 &	2 &	10 &	0.5 &	100 &	0.72& 0.0& 4.6	& 4.8&2&15 & {71} & 4& 0.85& 11.0\\[3pt]
4 &	5 &	10 &	0.5 &	100 &	4.6& 0.01& 25.2	& 4.8&2&19 & {53} & 207.5&21.8&1036.3 \\[3pt]
5 &	8&	10 &	0.05 &	100 &	0.01&0&0.02 &	2&2&2	& {3}	& 1.65&1.2&1.89\\[3pt]
5&	8&	10&	0.2 &	100 &	0.16&0.01& 1.93	& 2&2&3&	{7}&	6.1&1.9&11\\[3pt]
5 &	8&	10&	0.5&	{64} & 40& 0.03& 1023.4 &4.5& 2& 31 & {12} (*8) &	621.6&121.7&1025.7\\[3pt]
5 &	8 &	10 &	0.8 &	{44} &	167&0.18&1068.8	& 3&2&11 & \textcolor{red}{\bf 0} & x & x & x \\[3pt]	
5 &	8 &	10 &	0.95 &	{40} &  51&0.2& 1201 &	2.5&2&4	& \textcolor{red}{\bf 0} & x & x & x \\[3pt]
10 &	2 &	10 & 	0.5 &	99 &	1.65&0.0 & 7.74 &	24&2&104	& 86 (*1) &	42.1&1.6&339.7\\[3pt]
10 &	5 &	10 & 	0.5 &	91 &	16.5& 0.02& 86.4 &	22&2&407 &	{7} &	433.2& 23.2& 1111.2\\[3pt]
10 &	5 &	50 & 	0.05 &	99 &	6.7&0.08&76.5 &	3&2&12	& {32} (*1) &	70.3&9&452\\[3pt]
10& 	5 &	50 &	0.2 &	{8} &	354.8&0.2& 1429.1 &	2&2&3	& \textcolor{red}{\bf 0} & x & x & x \\[3pt]		
10& 	5 &	50 &	0.4 & 	{8} &	35.4&1.6&114.5 &	3&3&3	& \textcolor{red}{\bf 0} & x & x & x \\[3pt]		
10& 	5 & 	50 & 	0.5 & 	\textcolor{red}{\bf 0} & 	x & x & x &  x & x & x & \textcolor{red}{\bf 0} & x & x & x \\[3pt]	
    \hline
  \end{tabular}
\end{table*}
\subsection{Randomly Generated Networks}
First, we generate numerous randomly generated networks with $1$
output by varying the number of inputs ($n$), the number of hidden
layers ($k$), the number of neurons per hidden layer ($N$) and the
fraction of non-zero edge weights ($s$). Columns 1-4 of
Table~\ref{tab:microbenchmarks-1} shows the set of values chosen for
$\tupleof{n,k, N, s}$. For each chosen set, we generated $100$ random
neural networks. The zero weight edges were chosen by flipping a coin
with probability $1-s$ of choosing a weight to be zero. The nonzero
weights were uniformly chosen in the range $[-1,1]$. Next, we compare
our approach against Reluplex solver assuming that the inputs are in
the range $[-1,1]$. Each tool is run with a timeout of $20$
minutes. Table~\ref{tab:microbenchmarks-1} shows the comparison.
\sherlock outperforms Reluplex in terms of the number of examples
completed within the given timeout.  However, we note that many of the
Reluplex instances failed before the timeout due to an ``error''
reported by the solver. From the table we conclude that the total
number of neurons and the sparsity of the network have a large effect
on the overall performance. In particular, no approach is able to
handle dense examples with $250+$ neurons.

\begin{figure*}[t]
\centering
\includegraphics[width=5cm,height=5cm]{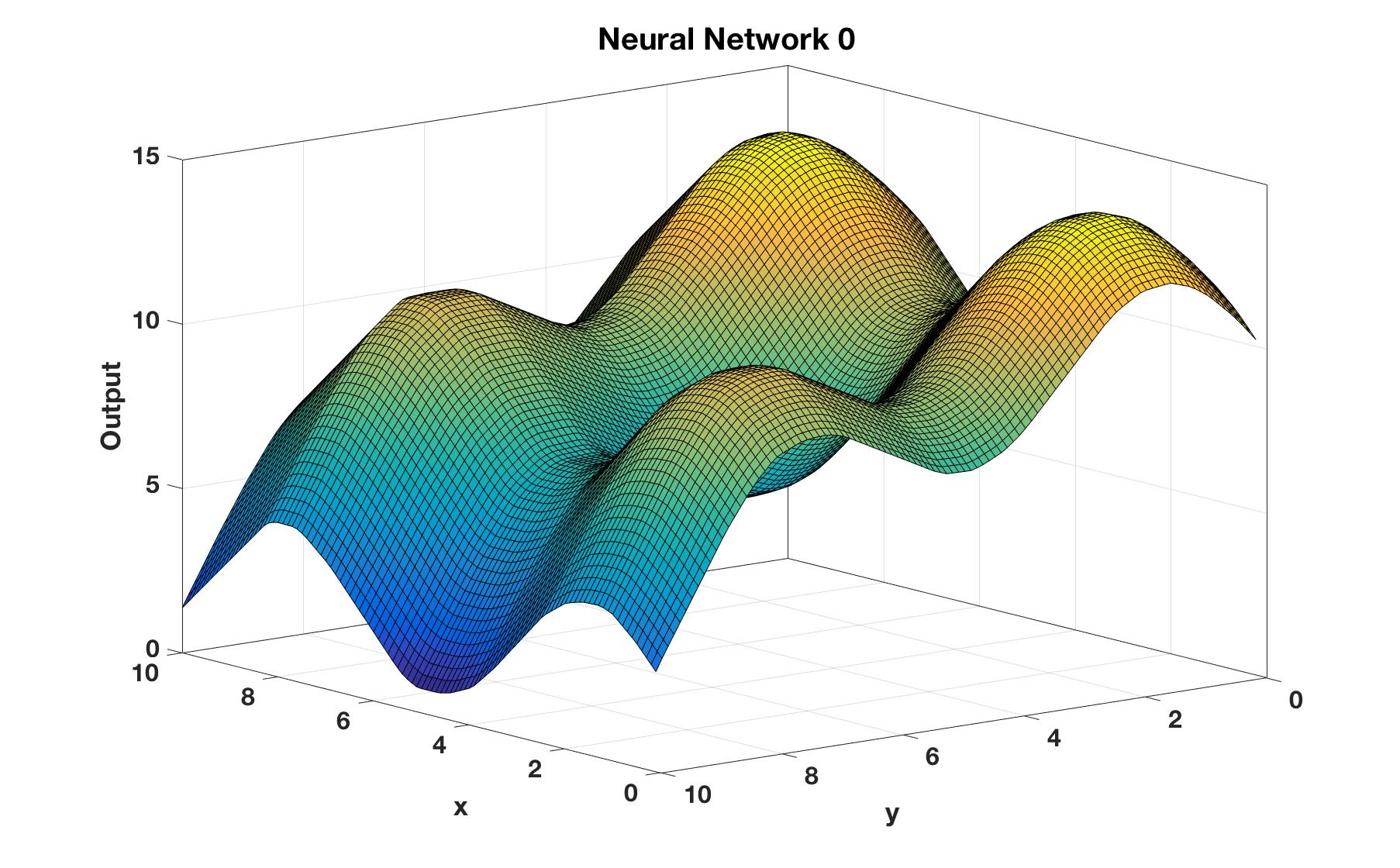}
\includegraphics[width=5cm,height=5cm]{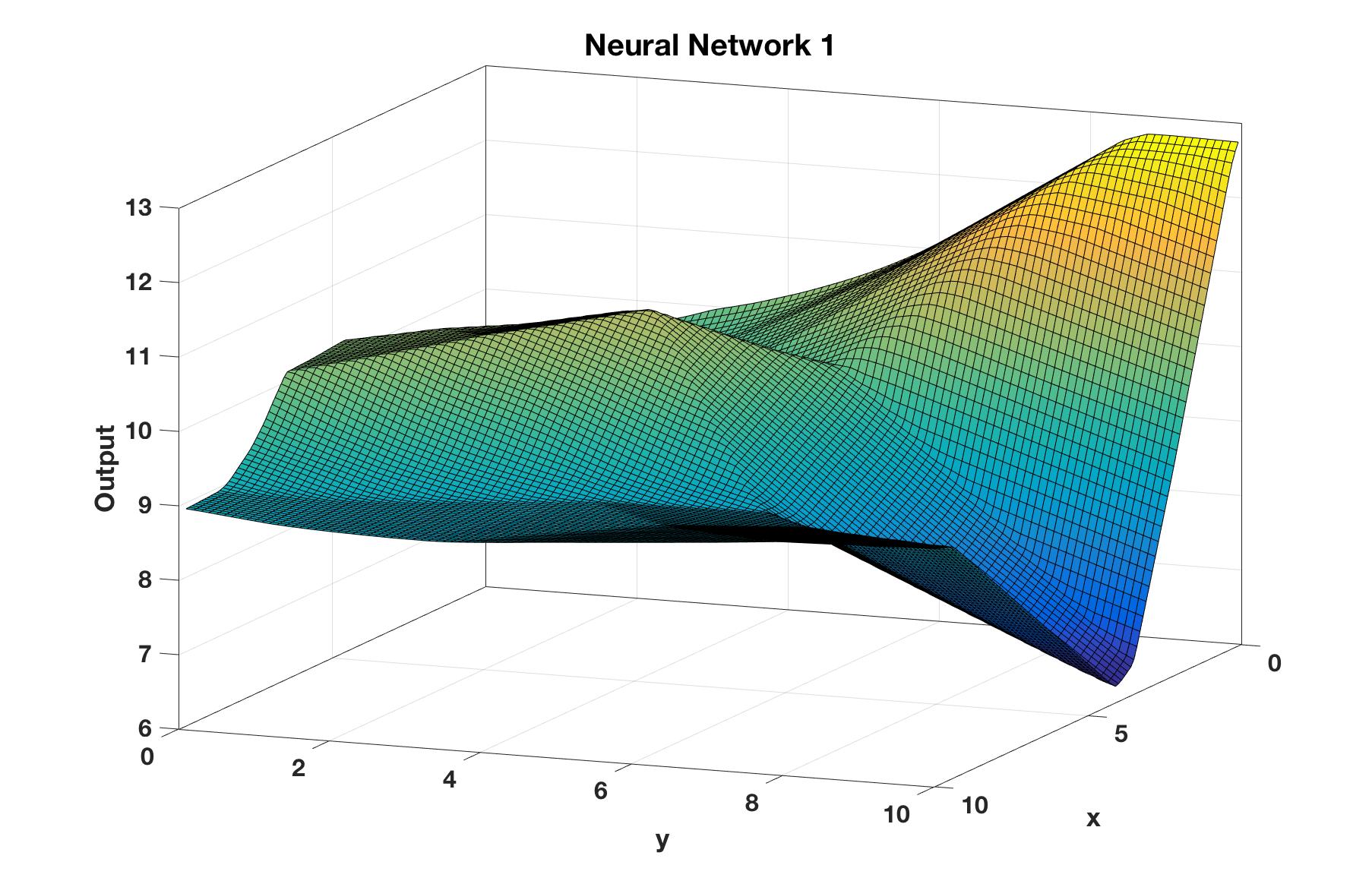}
\includegraphics[width=5cm,height=5cm]{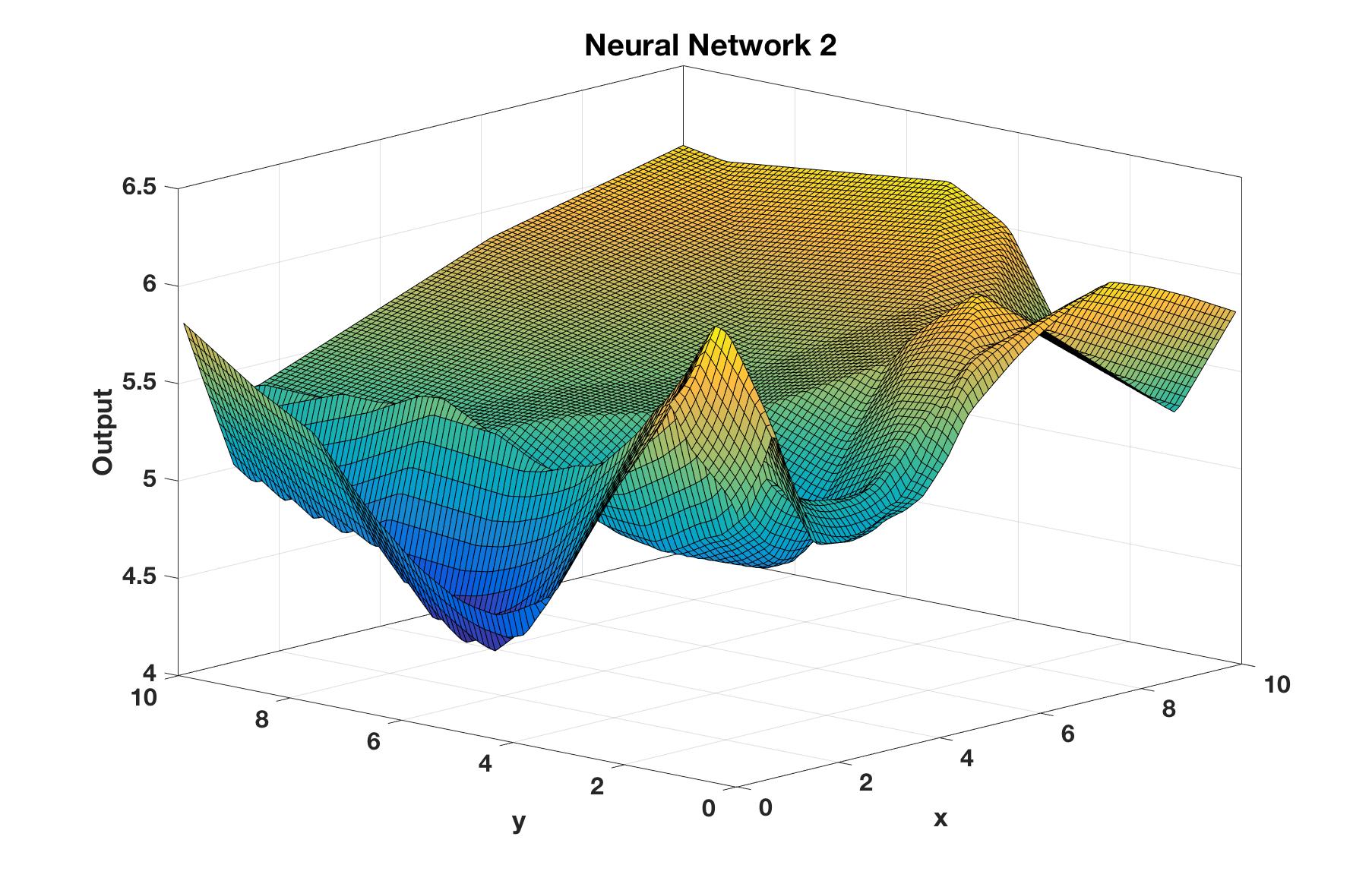}
\includegraphics[width=5cm,height=5cm]{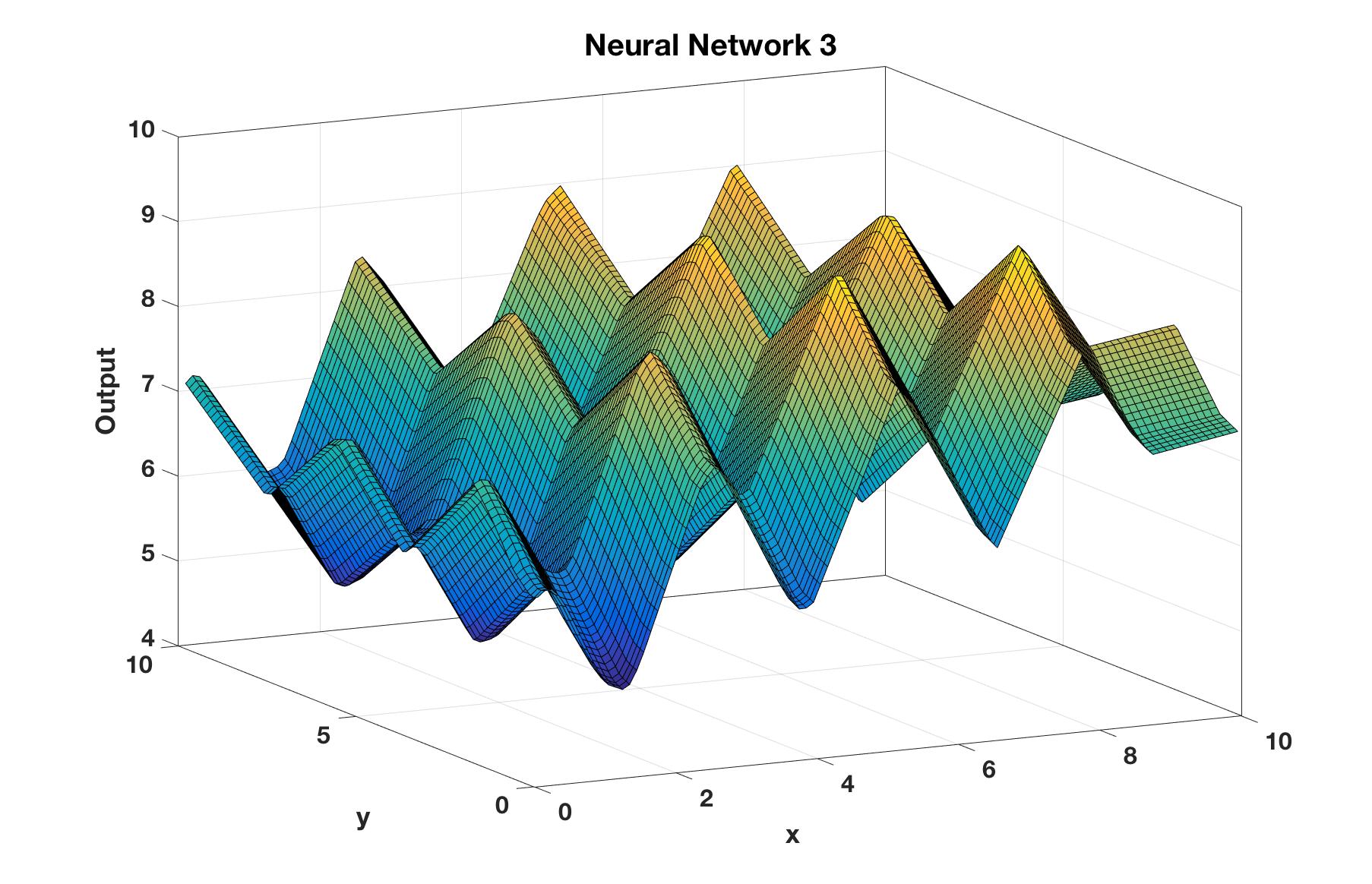}
\includegraphics[width=5cm,height=5cm]{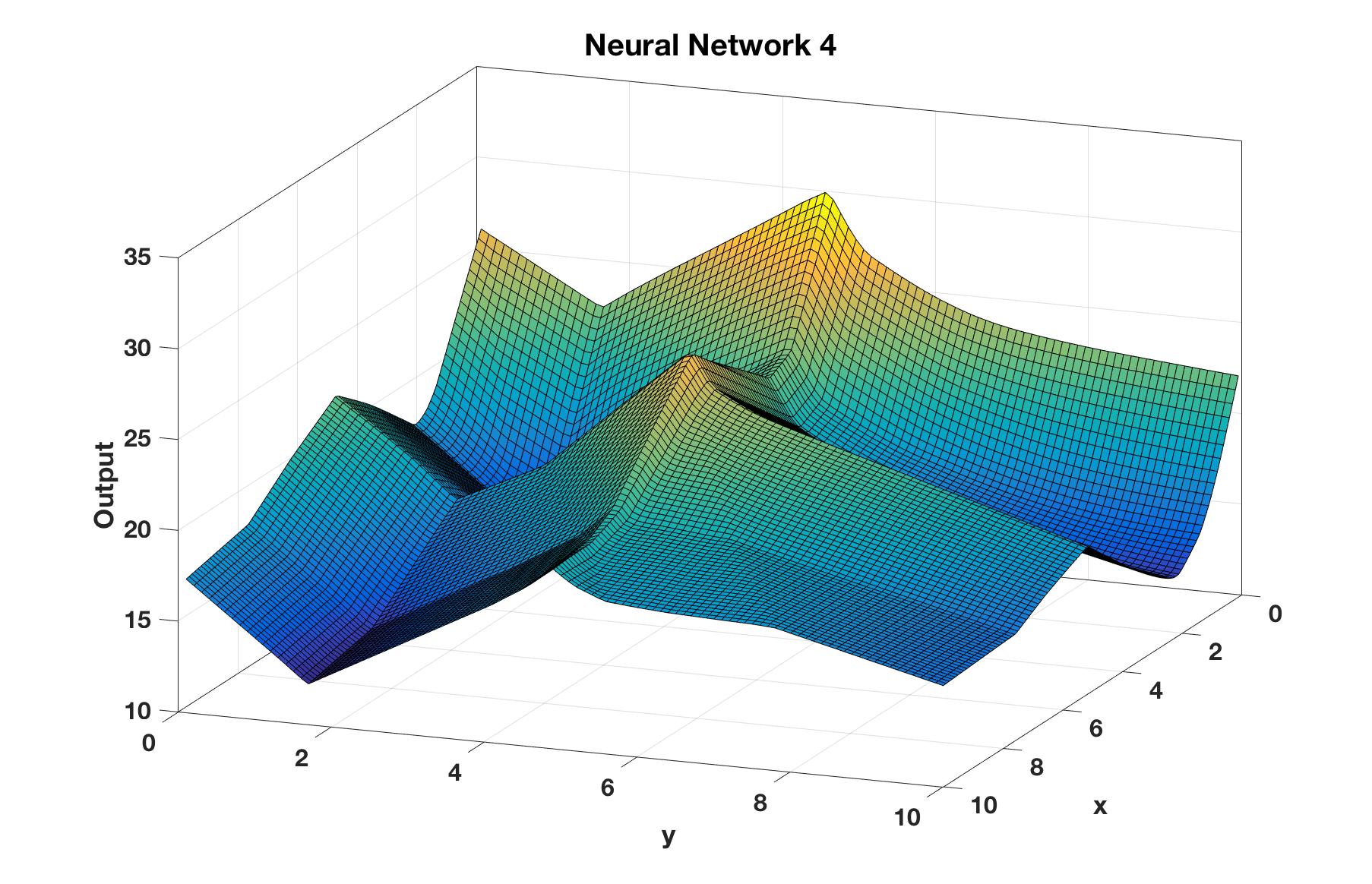}
\includegraphics[width=5cm,height=5cm]{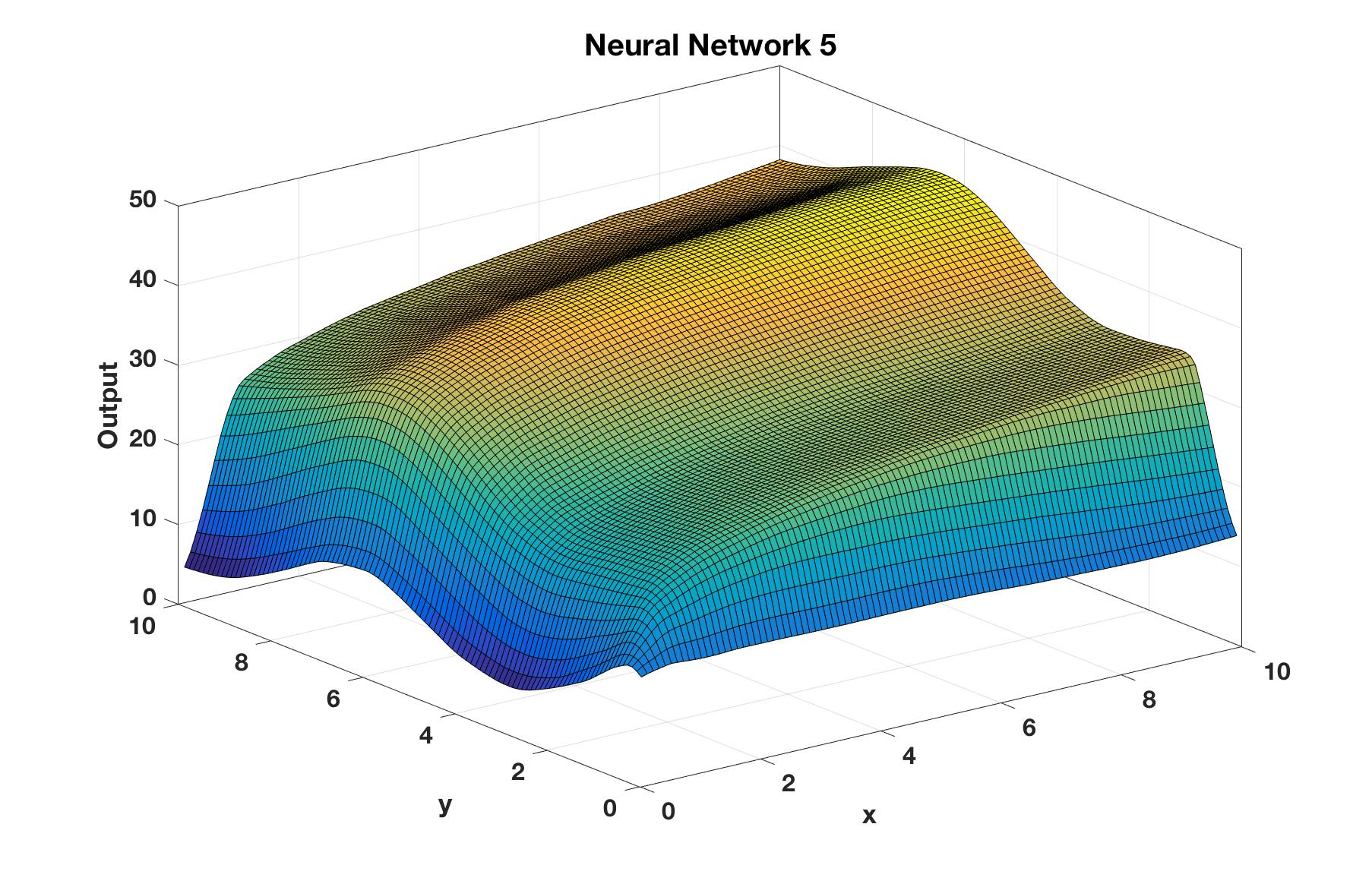}
\caption{Plots of the benchmark functions used to train the neural networks shown in Table~\ref{tab:microbenchmarks-2}.}
\label{fig_micro_benchmarks}
\end{figure*}

\subsection{Microbenchmarks on Known Functions}

Next, we consider a set of $6$ microbenchmarks that consist of
networks trained on known $2$ input, single output functions shown in
Figure~\ref{fig_micro_benchmarks}. These functions provide varying
numbers of local minima. We then sampled inputs/output pairs, and
trained a neural network with $1$ hidden layer for the first $5$
functions and $6$ hidden layer for the last. Next, we compared
\sherlock against the Reluplex solver in terms of the ability to
predict the output range of the function given the input range
$[-1,1]^2$ for the two inputs. 

Table~\ref{tab:microbenchmarks-2}
summarizes the results of our comparison. We note that in all cases,
our approach can solve the resulting problem within $10$
minutes. However, with a timeout of $1$ hour, Reluplex is unable to
deduce a range for $4$ out of $6$ instances.

\begin{table}[t]
  \caption{Performance results on networks trained from the $2$ input
    functions shown in
    Fig.~\ref{fig_micro_benchmarks}. \textbf{Legend:} $k$ number of
    layers, $N$: number of neurons/layer, $T_{s}$: Time taken by
    \sherlock, $T_{rplex}$: Time taken by
    Reluplex solver.  All results in this table were obtained on a
    Macbook Pro running ubuntu 16.04 with
   3 cores and 8 GB
    RAM.}\label{tab:microbenchmarks-2}
\begin{tabular}{ c c a c }
 \hline
 ID  & $k \times N$ & $T_{s}$  & $T_{rplex}$ \\
 \hline
$N_0$ & 1 $\times$ 100 & 1.9s &   1m 55s\\[3pt]
 $N_1$ & 1 $\times$ 200 & 2.4s &   13m 58s\\[3pt]
 $N_2$ & 1 $\times$ 500 & 17.8s &   timeout\\[3pt]
 $N_3$  & 1 $\times$ 500 & 7.6s &   timeout\\[3pt]
 $N_4$ & 1 $\times$ 1000 & 7m 57.8s &   timeout\\[3pt]
 $N_5$ & 6 $\times$ 250 & 9m 48.4s &   timeout\\[3pt]
 \hline
\end{tabular}
\end{table}

\subsection{Illustrative Applications}

We illustrate the use of \sherlock to infer properties of deep NNs for
important applications, starting with a neural network that controls a
nonlinear system. As mentioned earlier, such networks are increasingly
popular using deep policy learning~\cite{kahn2016plato}.

We trained a NN to control a nonlinear plant model (Example 7 from
\cite{control_example}) whose dynamics are describe by the ODE:
\[ \dot{x} = z^3 - y,\ \dot{y} = z,\ \dot{z} = u\,.\] We first devise
a model predictive control (MPC) scheme to stabilize this system to
the origin, and train the NN by sampling inputs from the state space
$X:\ [-0.5, 0.5]^3$ and using the MPC to provide the corresponding
control.

We trained a 3 input, 1 output network, with 2 hidden layers, having
300 neurons in the first layer and 200 neurons in the second layer.
For any state $(x,y,z) \in X$, we seek to know the range of the output
of the network to deduce the maximum/minimum control input. This is a
direct application of the range estimation problem. \sherlock is able
to deduce a tight range for the control input: $[-2.68772, 3.00263]$.

\subsection{Handwriting Recognition Networks} 
We illustrate an application of range estimation to large neural networks involved in
pattern classification, wherein, we wish to infer classification labels for
images. The MNIST handwriting recognition data set has emerged as a prototype
application~\cite{mnist}.

Given a neural network classifier $N$, and a given image input $I$, we
wish to explore a set of possible perturbations to the image to find a
nearby input that can alter the label that $N$ provides to the
perturbed image. Alternatively, given a space of possible
perturbations, we wish to prove that the network $N$ is \emph{robust}
to these perturbations: i.e, no perturbation can alter the label that
$N$ provides to the perturbed image.

We trained a NN with $28 \times 28$ inputs, and $10$ outputs, and $3$
hidden layers with $200, 100$ and $50$ neurons in the layers
respectively on the MNIST digit recognition data set~\cite{mnist}.  The
network has $10$ outputs that are then fed to a ``softmax'' output
layer to provide the final classification. For the purposes of our
application, we remove this softmax layer and examine these $10$
outputs that represent numerical scores corresponding to the digits
$0-9$.

We first use our approach to generate adversarial perturbations to the
images, which can flip the output. To do so, we define for each image
a set of pixels and a range of possible perturbation of these pixels.
This forms a polyhedron that we will call the perturbation space.
Next, we use our approach to explore points in this perturbation
space, stopping as soon as we obtain a point with a different label.
This is achieved simply using local search iterations. 
 Figure~\ref{fig_controlled_pert}, demonstrates such
an effect, we started with images that are classified correctly
as the digits ``0'' and ``1''. The perturbations are applied
to selected pixels resulting in  the digit ``0''  labeled as
``8'',  and ``1'' labeled as ``2''.

\begin{figure}[h!]
\begin{center}
\includegraphics[scale = 0.2]{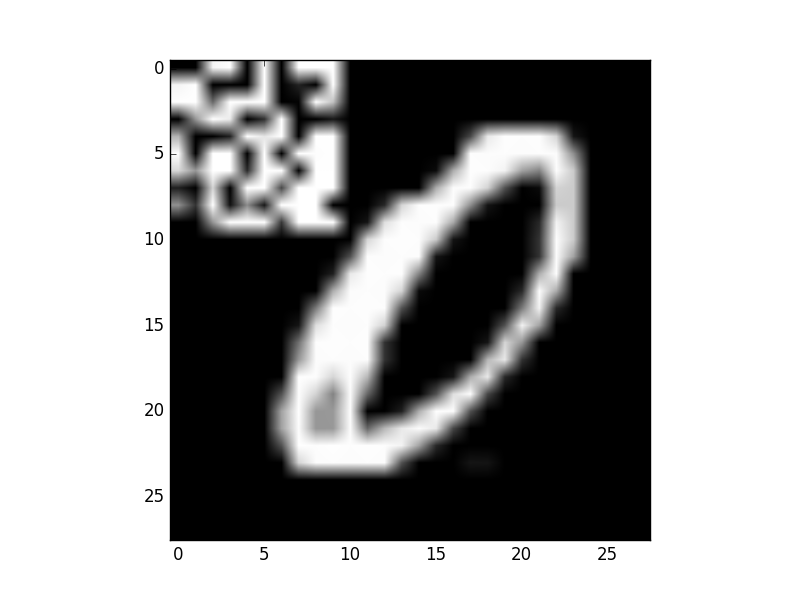} 
\includegraphics[scale = 0.2]{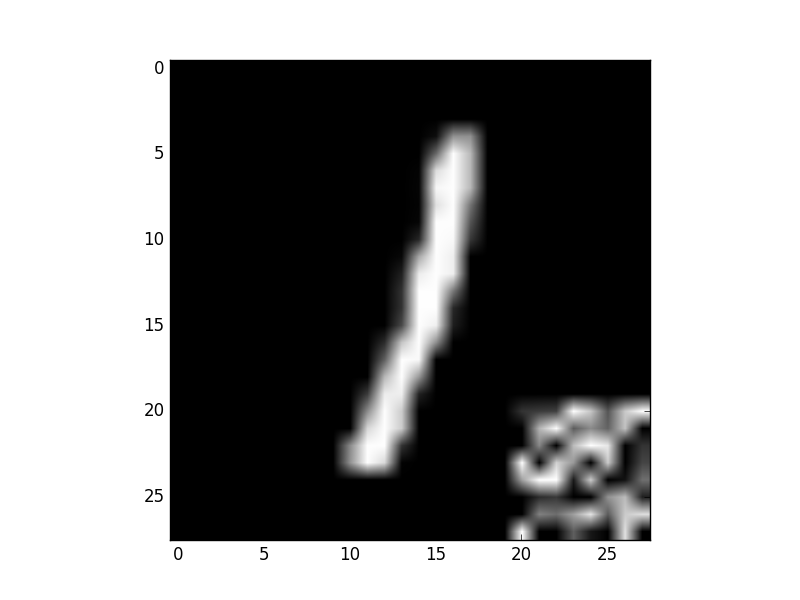} 
\end{center}
\caption{(\textbf{Left}) Perturbing pixels in the top right corner of a  28 $\times$ 28 image changes the network's output  label from $0$ to $8$. (\textbf{Right}) Perturbation changes label from $1$ to $2$.}
\label{fig_controlled_pert}
\end{figure}

\begin{figure}[h!]
\begin{equation*}
\begin{array}{rl|rl}
range(l_0) &\subseteq [0.0, 0.763] & range(l_1) &\subseteq  [ 0, 0 ] \\
range(l_2) &\subseteq [0, 0] &  range(l_3) &\subseteq [ 3.365, 5.71] \\
range(l_4) &\subseteq [0,0] & \color{red}range(l_5) &\subseteq \color{red}[9.62, 13.4] \\ 
range(l_6) &\subseteq [1.24, 2.25] & range(l_7) &\subseteq [0.016, 0.592] \\
range(l_8) &\subseteq [2.31, 3.5] & range(l_9) &\subseteq [3.39, 5.3] \\
\end{array}
\end{equation*}
\caption{Output range computed by \sherlock for all possible perturbations of $5\times 10$ block as shown in Figure~\ref{fig_robust_pert}.}\label{Fig:sherlock-image-ranges}
\end{figure}

\begin{figure}[h!]
\centering
\includegraphics[scale=0.3]{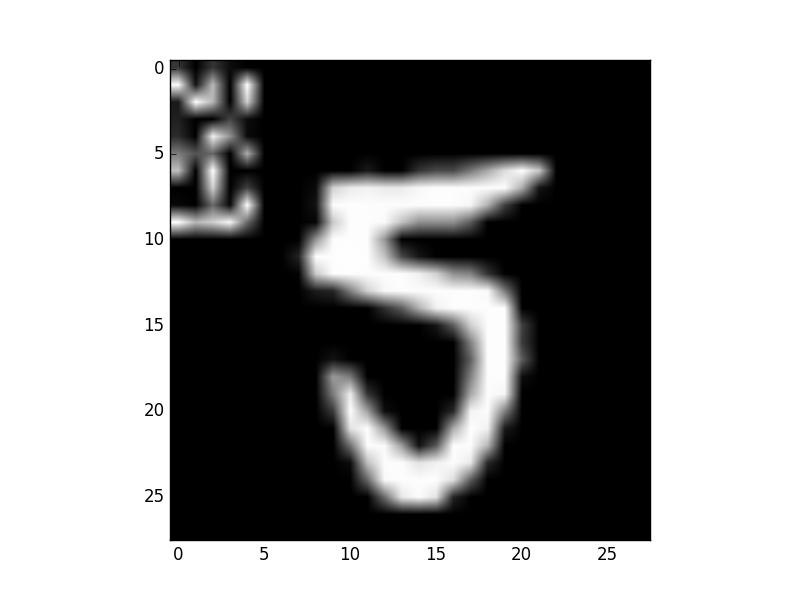} 
\caption{The image above shows the effect of perturbation on an image that the trained NN is robust to. The perturbation were only allowed to effect a block of size $5 \times 10$ on the top left of the image}
\label{fig_robust_pert}
\end{figure} 

However, rather finding a perturbation, we wish to \emph{prove}
robustness to a class of perturbations. I.e, for all possible
perturbations, the label ascribed by the network to an image
remains the same.

We chose an image labeled as $5$ and specified a set of perturbations to
a block of size $5 \times 10$ on the top left corner of the image. One
such perturbation is shown in Figure~\ref{fig_robust_pert}. We used
\sherlock to compute a range over the output labels for all such perturbations.
Figure~\ref{Fig:sherlock-image-ranges} shows the output ranges for each label.
It is clear that all perturbations will lead the softmax layer to choose the label
$5$ over all other labels for the image.

\section{Conclusion}\label{sec:conc}
Thus, we have presented a combination of local and global search for
estimating the output ranges of neural networks given constraints on
the input. Our approach has been implemented inside the tool \sherlock
and compared the results obtained with the solver Reluplex. We have
also demonstrated the application of our approach to interesting
applications to control and classification problems.

In the future, we wish to improve \sherlock in many directions,
including the treatment of recurrent neural networks, handling
activation functions beyond ReLU units and providing faster
alternatives to the MILP for global search.

\bibliographystyle{aaai}
\bibliography{references}

\end{document}